\documentclass[10pt,twocolumn]{article}

\usepackage[margin=2.5cm]{geometry}
\usepackage{amsmath,amssymb,amsthm}
\usepackage{graphicx}
\usepackage{booktabs}
\usepackage{hyperref}
\usepackage{xcolor}
\usepackage{microtype}
\usepackage{enumitem}
\usepackage{caption}
\usepackage{subcaption}
\usepackage{algorithm}
\usepackage{algpseudocode}
\usepackage{natbib}
\usepackage{tikz}
\usetikzlibrary{arrows.meta, positioning, shapes.geometric, fit, calc}

\hypersetup{
  colorlinks=true,
  linkcolor=blue!70!black,
  citecolor=green!50!black,
  urlcolor=blue!70!black
}
\newtheorem{theorem}{Theorem}
\newtheorem{proposition}{Proposition}

\newtheorem{remark}{Remark}

\newcommand{\R}{\mathbb{R}}
\newcommand{\abs}[1]{\left|#1\right|}

\DeclareMathOperator{\rfft}{RFFT}

\title{%
  \textbf{S2P-Net: A Spectral-Spatial Polar Network for\\
  Rotation-Invariant Object Recognition in Low-Data Regimes}%
}

\author{%
  Albert Heruth\\
  \texttt{Unaffiliated Researcher}\\
  \texttt{Heide, Schleswig-Holstein, Germany}\\
  \texttt{albert.heruth@gmail.com}
}

\date{June 2026}

\begin{document}
\maketitle

\begin{abstract}
We present \textbf{S2P-Net} (Spectral-Spatial Polar Network), a compact classifier that achieves mathematically guaranteed rotation invariance without data augmentation. The idea rests on a classical Fourier property: rotating an image merely shifts the phase of its frequency representation, leaving the magnitude spectrum unchanged. S2P-Net maps the image to polar coordinates, takes the one-dimensional Real FFT along the angular axis, and pools the magnitude into a 64-dimensional feature vector that a small MLP then classifies. The three feature-extraction stages are parameter-free, so the whole model carries only \textbf{6,564 trainable parameters}---a $323\times$ reduction over our CNN baseline. On an 80-image, four-class industrial dataset, both models reach 100\% accuracy at all 12 test orientations when rotation-augmented data is available. In a low-data setting, however (3 images per class, no rotation augmentation), S2P-Net holds $71.2\%$ mean accuracy with a standard deviation of only $1.6\%$ across angles, while the CNN averages $60.0\%$ ($\sigma=22.9\%$) and collapses to $19.1\%$ at $180^\circ$. We then extend the same shift argument to \emph{scale} via log-polar sampling---the log-polar variant is the most scale-stable of three models over five seeds---show that a lightweight \emph{centroid front-end} removes the method's one structural weakness, object centering, restoring a flat $93.3\%$ under de-centering, and map its robustness envelope, where the global spectral representation proves competitive under noise but, by construction, more occlusion-sensitive than a CNN. In short: mathematical inductive bias can substitute for data.
\end{abstract}

\section{Introduction}
\label{sec:intro}

Object recognition in industrial settings---parts sorting, pick-and-place robotics, quality inspection---frequently faces two simultaneously difficult constraints. First, training data is scarce: collecting and labeling thousands of images for every new part type is expensive and time-consuming. Second, objects appear on conveyor belts or work surfaces in arbitrary orientations that cannot be controlled. Standard convolutional neural networks (CNNs) address the second constraint through rotation augmentation during training \citep{lecun1998gradient,krizhevsky2012imagenet}, effectively forcing the network to memorize each object at many orientations. This approach is well-suited when large datasets are available, but in low-data regimes---where only a handful of images exist per class---the network sees each orientation at most once or not at all, leading to catastrophic performance degradation at unseen rotations.

The ideal solution is a representation that is \emph{invariant} to rotation by construction, so that the classifier never needs to see rotated examples. A principled route to such invariance is offered by group theory \citep{cohen2016group,weiler2019general}: by designing feature maps that transform equivariantly under the action of a symmetry group, and then pooling over that group to obtain an invariant summary, one can guarantee that the final representation is independent of orientation. Existing group-equivariant architectures, however, typically retain much of the parameter complexity of standard CNNs and require non-trivial modifications to the convolutional stack.

In this work we take a different, arguably more direct route. The \emph{Fourier rotation theorem} states that rotating an image by an angle $\phi$ multiplies each Fourier coefficient by a complex exponential: $\hat{I}_\phi(k) = \hat{I}(k)\,e^{-jk\phi}$. As a consequence, the \emph{magnitude} $|\hat{I}(k)|$ is completely unaffected by rotation. If we can translate an image into a representation whose axes correspond to Fourier frequencies, the rotation degree of freedom is absorbed into the phase and discarded automatically.

Polar coordinates provide exactly this translation. In polar coordinates $(r,\theta)$ centered on the object, a rotation of the object becomes a cyclic shift along the $\theta$-axis. Applying a 1-D Discrete Fourier Transform along $\theta$ converts that cyclic shift into a per-frequency phase offset, and taking the magnitude spectrum yields a feature map that is provably rotation-invariant.

S2P-Net realises this pipeline as a sequence of four deterministic, parameter-free signal-processing steps followed by a tiny trainable MLP. The first three layers contain \emph{zero trainable parameters}; all learning is concentrated in the 6,564-parameter classifier head. This extreme parameter efficiency makes S2P-Net naturally resistant to overfitting on small datasets, while the inductive bias of the Fourier magnitude spectrum ensures rotation invariance at test time without any rotated training samples.

\paragraph{Contributions.}
\begin{itemize}[leftmargin=1.2em]
  \item We design S2P-Net, a rotation-invariant architecture based on polar-domain spectral analysis with only 6,564 trainable parameters.
  \item We provide a formal proof that the spectral representation extracted by S2P-Net is invariant under arbitrary 2-D image rotation.
  \item We experimentally demonstrate that S2P-Net outperforms a standard CNN by $11.2$ percentage points in mean accuracy under the low-data, no-augmentation setting (3 training images per class), while the CNN collapses by up to $70.8$ pp at individual angles.
  \item We show that both architectures attain perfect accuracy when sufficient augmented data is available, confirming that S2P-Net's advantage is specifically in the low-data regime.
  \item We extend the architecture to \emph{scale} invariance via log-polar sampling, prove that scaling becomes a radial shift absorbed by the pooling stage, and demonstrate over five seeds that the log-polar variant is the most scale-stable of the three models tested.
  \item We provide a five-seed centering ablation that quantifies S2P-Net's sole architectural assumption---object centering---and show that a moment-based centroid front-end removes the dependence, restoring centred-baseline accuracy at every tested offset.
  \item We characterise the method's robustness envelope under noise, occlusion and blur, finding that the global spectral representation is competitive on noise but, by construction, more occlusion-sensitive than a CNN---a complementary rather than dominant robustness profile.
\end{itemize}

\section{Related Work}
\label{sec:related}

\paragraph{Rotation-invariant and equivariant CNNs.}
Cohen and Welling \citep{cohen2016group} introduced Group-Equivariant CNNs (G-CNNs), which extend standard convolutions to act equivariantly under discrete rotation groups such as $p4$ and $p4m$. Subsequent work generalised this framework to continuous rotation groups \citep{weiler2018learning,weiler2019general}, steerability \citep{worrall2017harmonic}, and general Lie groups \citep{finzi2020generalizing}. While theoretically elegant, these methods modify the convolutional operator itself and generally maintain a parameter count comparable to conventional deep networks.

\paragraph{Polar-domain approaches.}
The connection between polar coordinates, the Fourier transform, and rotation invariance has been exploited in classical computer vision \citep{zahn1972fourier,hu1962visual}. Polar Transformer Networks \citep{esteves2018polar} learn to compute polar transforms as part of a spatial transformer framework, but require trainable parameters for the transformation itself. Log-polar networks \citep{sosnovik2020scale} achieve joint scale and rotation equivariance at the cost of an expanded parameter space. Our approach differs in that the polar transform and subsequent FFT are \emph{fixed, non-trainable} operations derived from Fourier theory, and the only learned component is a small downstream classifier.

\paragraph{Symmetry-based feature extraction.}
Dieleman et al.\ \citep{dieleman2015exploiting} demonstrated rotation invariance on galaxy morphology classification by averaging predictions over discrete rotation orbits. This test-time augmentation approach avoids modifying the architecture but multiplies inference cost by the number of orientations tested. Marcos et al.\ \citep{marcos2017rotation} proposed Rotation Equivariant Vector Field Networks by computing maximum responses over a set of rotated filters. S2P-Net instead achieves invariance at the feature-extraction stage with no additional inference overhead.

\paragraph{Low-data and few-shot learning.}
Prototypical Networks \citep{snell2017prototypical} and Matching Networks \citep{vinyals2016matching} address few-shot classification (generalisation to unseen classes from few examples) through metric learning and episodic training. Our setting is complementary but distinct: we fix the class set and reduce the per-class training count to an extreme minimum (3 images per class), focusing on orientation robustness rather than inter-class generalisation. The rotation-invariant representation produced by S2P-Net could in principle be combined with metric-learning frameworks to tackle both challenges jointly.

\section{Mathematical Foundation}
\label{sec:math}

\subsection{Polar coordinate transform}
\label{subsec:polar}

Let $I: \Omega \to \R$ be a grayscale image defined on a discrete domain $\Omega \subset \R^2$, centred at the image centre $(c_x, c_y)$. The polar representation $\tilde{I}: [0, R_{\max}]\times[0, 2\pi) \to \R$ is defined by
\begin{equation}
  \tilde{I}(r, \theta) \;=\; I\!\left(c_x + r\cos\theta,\; c_y + r\sin\theta\right),
\label{eq:polar}
\end{equation}
where $r$ is the radial distance from the centre and $\theta \in [0, 2\pi)$ is the polar angle. In our implementation, $\tilde{I}$ is sampled on a uniform grid of $R \times \Theta$ points and evaluated via bilinear interpolation.

\paragraph{Effect of image rotation.}
Let $I_\phi$ denote the image $I$ rotated counter-clockwise by angle $\phi$. In polar coordinates,
\begin{align}
  \tilde{I}_\phi(r, \theta)
    &\;=\; I_\phi\!\left(c_x + r\cos\theta,\; c_y + r\sin\theta\right) \notag\\
    &\;=\; \tilde{I}(r, \theta - \phi).
\label{eq:rotation_shift}
\end{align}
Rotation of the original image becomes a \emph{cyclic shift} along the $\theta$-axis of the polar representation.

\subsection{Harmonic decomposition and rotation invariance}
\label{subsec:fourier}

For each fixed radius $r$, we treat the angular profile $\tilde{I}(r, \cdot)$ as a $2\pi$-periodic signal and decompose it into Fourier series coefficients:
\begin{equation}
  \mathcal{F}(r, k) \;=\; \int_0^{2\pi} \tilde{I}(r, \theta)\; e^{-jk\theta}\, d\theta,
  \quad k \in \mathbb{Z}.
\label{eq:fourier}
\end{equation}

\begin{theorem}[Rotation invariance of the magnitude spectrum]
\label{thm:invariance}
Let $\mathcal{F}_\phi(r, k)$ denote the Fourier coefficient of the rotated signal
$\tilde{I}_\phi(r, \cdot)$. Then
\begin{equation}
  \abs{\mathcal{F}_\phi(r, k)} \;=\; \abs{\mathcal{F}(r, k)}
  \quad \forall\, k \in \mathbb{Z},\; \forall\, r,\, \phi.
\label{eq:invariance}
\end{equation}
\end{theorem}

\begin{proof}
Using \eqref{eq:rotation_shift}, the Fourier coefficient of the rotated profile is
\begin{equation}
  \mathcal{F}_\phi(r, k)
    = \int_0^{2\pi} \tilde{I}(r,\theta-\phi)\,e^{-jk\theta}\,d\theta.
\end{equation}
Substituting $\theta' = \theta - \phi$ and invoking $2\pi$-periodicity of $\tilde{I}(r,\cdot)$:
\begin{align}
  \mathcal{F}_\phi(r, k)
    &= \int_0^{2\pi} \tilde{I}(r,\theta')\,e^{-jk(\theta'+\phi)}\,d\theta'
     = e^{-jk\phi}\,\mathcal{F}(r, k).
\end{align}
Taking the complex modulus, $\abs{e^{-jk\phi}} = 1$, gives
$\abs{\mathcal{F}_\phi(r, k)} = \abs{\mathcal{F}(r, k)}$.
\end{proof}

\begin{remark}[Discrete implementation]
In practice the angular dimension is sampled at $\Theta = 128$ uniform points and the
1-D Real FFT is computed on these samples. For rotation angles $\phi = 2\pi m/\Theta$
($m \in \mathbb{Z}$), the continuous shift maps exactly to an integer sample displacement
and the theorem holds without modification. For general $\phi$, bilinear resampling
introduces a bounded interpolation error. The near-constant accuracy profile in
Table~\ref{tab:few_shot_angles} ($\sigma = 1.6\%$ across all 12 test angles) empirically
confirms that this residual is negligible for the task at hand.
\end{remark}

Theorem~\ref{thm:invariance} establishes that the magnitude spectrum $M(r,k) = |\mathcal{F}(r,k)|$ is \emph{strictly invariant} to any in-plane rotation of the object in the continuous setting. This invariance is not learned; it follows directly from the Fourier shift property and holds for every possible input image, provided the object is centred in the image.

\subsection{Harmonic interpretation}
The frequency index $k$ corresponds to the $k$-fold rotational symmetry order of the object:
\begin{itemize}[leftmargin=1.2em]
  \item $k=0$: mean radial intensity (DC component).
  \item $k=4$: energy in 4-fold symmetric patterns (e.g., square washers, cubes).
  \item $k=6$: energy in 6-fold symmetric patterns (e.g., hexagonal nuts).
\end{itemize}
Different object classes excite different harmonics, providing a physically interpretable and discriminative fingerprint.

\subsection{Extension to scale invariance via log-polar sampling}
\label{subsec:logpolar}
Object scale is a second nuisance transformation in industrial imaging: the same part may appear larger or smaller depending on its height on the work surface or the camera's field of view. The polar construction extends to scale by a single change of variable. Let the radial axis be sampled \emph{logarithmically}, $u = \log r$, so that the image is represented as $\tilde{I}(u,\theta)$.

\begin{proposition}[Scale becomes a radial shift]
\label{prop:scale}
A uniform scaling of the centred object by a factor $s>0$, $I_s(\mathbf{p}) = I(\mathbf{p}/s)$, maps the log-polar representation to
\begin{equation}
  \tilde{I}_s(u,\theta) = \tilde{I}(u - \log s,\; \theta),
\end{equation}
i.e.\ a pure translation by $\log s$ along the $u$-axis. Combined with Eq.~\eqref{eq:rotation_shift}, a simultaneous rotation by $\phi$ and scaling by $s$ act as the joint 2-D translation $(u,\theta)\mapsto(u+\log s,\ \theta-\phi)$.
\end{proposition}

\begin{proof}
A point at radius $r$ in $I$ appears at radius $sr$ in $I_s$; in log coordinates $\log(sr)=\log s + \log r$, hence the radial index is shifted by the constant $\log s$ independently of $\theta$. The angular argument is unchanged, and the rotation result of Eq.~\eqref{eq:rotation_shift} applies verbatim to the $\theta$-axis.
\end{proof}

The angular Fourier magnitude of Theorem~\ref{thm:invariance} already removes the $\theta$-translation. The remaining $u$-translation is absorbed by the spectral pooling of Section~\ref{subsec:stage3}: the global mean and global maximum over $u$ are invariant under any cyclic shift of $u$, and approximately invariant under a non-cyclic shift \emph{provided the object's support stays within the sampled window} $[u_{\min},u_{\max}]$. Consequently the S2P-Net feature is jointly rotation- and (window-limited) scale-invariant when log-polar sampling is used. Unlike rotation---which is exactly periodic---scale invariance is bounded: an object scaled until it exceeds the radial window (or shrinks below $r_{\min}$) loses energy to truncation, so the guarantee holds over a finite scale band rather than globally. Section~\ref{subsec:scale_results} measures this band empirically.

\section{S2P-Net Architecture}
\label{sec:arch}

S2P-Net processes a centred, square greyscale image $\mathbf{x} \in \R^{H \times H}$ through four sequential stages. The first three stages are parameter-free signal processing operations; only the fourth stage contains trainable parameters.

Figure~\ref{fig:architecture} gives an overview of the pipeline.

\begin{figure}[t]
\centering
\begin{tikzpicture}[
  node distance=0.45cm and 0.1cm,
  block/.style={rectangle, draw, rounded corners=3pt, fill=blue!10,
                text width=3.5cm, minimum height=0.8cm, align=center,
                font=\small},
  arrow/.style={-{Stealth[length=5pt]}, thick},
  label/.style={font=\scriptsize\itshape, text=gray}
]
  \node[block] (img) {Input Image\\$(B,1,128,128)$\\greyscale, centred};
  \node[block, below=of img] (polar) {Polar Transform\\$(B,1,64,128)$\\$r \times \theta$ grid};
  \node[block, below=of polar] (fft) {Harmonic Signature\\$(B,1,64,32)$\\$|\text{RFFT along }\theta|$};
  \node[block, below=of fft] (pool) {Spectral Pooling\\$(B,64)$\\mean + max over $r$};
  \node[block, below=of pool] (mlp)  {MLP Classifier\\$(B,C)$\\$64{\to}64{\to}32{\to}C$};

  \draw[arrow] (img)   -- node[label, right=2pt]{bilinear sampling} (polar);
  \draw[arrow] (polar) -- node[label, right=2pt]{1-D RFFT, modulus} (fft);
  \draw[arrow] (fft)   -- node[label, right=2pt]{global pooling}    (pool);
  \draw[arrow] (pool)  -- node[label, right=2pt]{trainable}         (mlp);

  \node[label, right=0.2cm of polar] (p0) {0 params};
  \node[label, right=0.2cm of fft]   (p1) {0 params};
  \node[label, right=0.2cm of pool]  (p2) {0 params};
  \node[label, right=0.2cm of mlp]   (p3) {\textbf{6,564 params}};
\end{tikzpicture}
\caption{S2P-Net pipeline. Three deterministic, parameter-free stages extract a rotation-invariant feature vector; a lightweight MLP performs classification.}
\label{fig:architecture}
\end{figure}
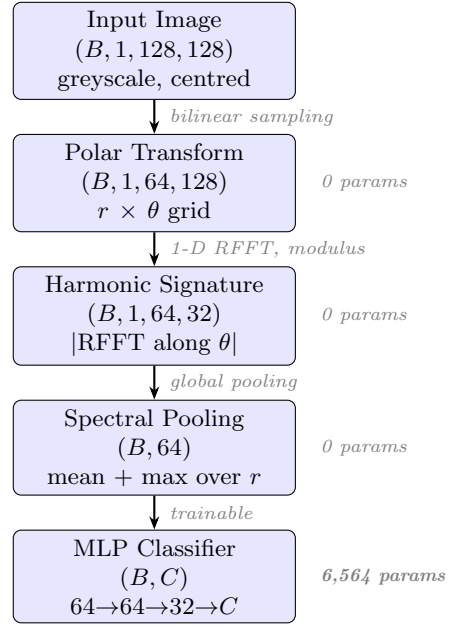

\subsection{Stage 1: Polar Transform Layer}
\label{subsec:stage1}

The input image $\mathbf{x} \in \R^{1 \times H \times H}$ is resampled onto a polar grid of $R \times \Theta$ positions. For each point $(r_i, \theta_j)$ in the output grid, the corresponding Cartesian position is
\begin{equation}
  x_{ij} = c_x + r_i \cos\theta_j, \quad y_{ij} = c_y + r_i \sin\theta_j,
\end{equation}
with $r_i$ sampled uniformly in $[0, H/2]$ and $\theta_j$ sampled uniformly in $[0, 2\pi)$. Pixel values at $(x_{ij}, y_{ij})$ are obtained via bilinear interpolation using PyTorch's \texttt{grid\_sample} with zero-padding outside the image boundary. The sampling grid is precomputed once at initialisation and stored as a non-trainable buffer, making the forward pass computationally efficient.

\textbf{Parameters: 0.} Output shape: $(B, 1, R, \Theta)$ with $R{=}64$, $\Theta{=}128$.

\paragraph{Log-polar variant.} The radial samples $r_i$ may be placed either uniformly in $[0, H/2]$ (the default) or logarithmically in $[r_{\min}, H/2]$ with $r_{\min}{=}2$ pixels. The logarithmic placement realises Proposition~\ref{prop:scale} and is the only change required to obtain the scale-invariant model evaluated in Section~\ref{subsec:scale_results}; the parameter count is unchanged.

\subsection{Stage 2: Harmonic Signature Layer}
\label{subsec:stage2}

For each radius bin $r_i$, the angular profile $\tilde{\mathbf{x}}(r_i, \cdot) \in \R^{\Theta}$ is transformed via the 1-D Real FFT:
\begin{equation}
  \mathbf{c}(r_i) = \rfft\!\left[\tilde{\mathbf{x}}(r_i, \cdot)\right] \in \mathbb{C}^{\Theta/2+1}.
\end{equation}
The magnitude $M(r_i, k) = |\mathbf{c}(r_i)[k]|$ is retained for the first $K_{\max} = 32$ frequency bins. By Theorem~\ref{thm:invariance}, this tensor is invariant to any rotation of the original image.

\textbf{Parameters: 0.} Output shape: $(B, 1, R, K)$ with $K{=}32$.

\subsection{Stage 3: Spectral Pooling}
\label{subsec:stage3}

The spectral tensor $M \in \R^{B \times 1 \times R \times K}$ is reduced to a fixed-length feature vector by applying global mean pooling and global max pooling across the radius dimension:
\begin{equation}
  \mathbf{f}_k = \left[\frac{1}{R}\sum_{r} M_{r,k},\;\; \max_{r} M_{r,k}\right] \in \R^{2K}.
\end{equation}
Concatenating both pooling statistics yields a $2K = 64$-dimensional vector that captures both the average symmetry content and the peak symmetry response at each harmonic order. Figure~\ref{fig:featurevec} visualises this vector for a hexagonal nut and confirms that it is, as Theorem~\ref{thm:invariance} predicts, essentially unchanged when the object is rotated (a relative $L_1$ deviation of at most $3.4\%$ across seven test angles, attributable to bilinear-resampling error alone).

\textbf{Parameters: 0.} Output shape: $(B, 64)$.

\begin{figure*}[t]
\centering
\includegraphics[width=0.92\textwidth]{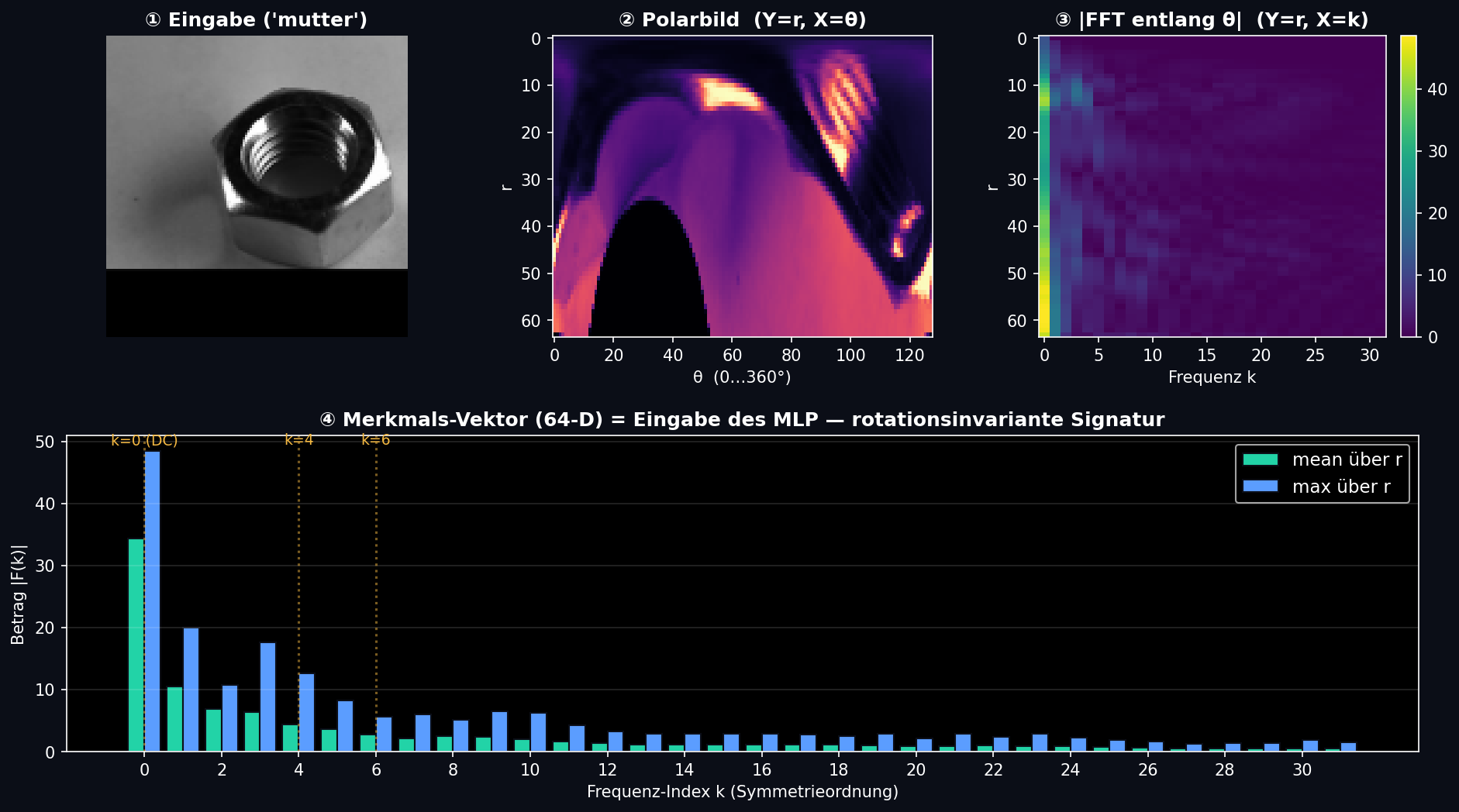}
\caption{The complete parameter-free pipeline for one hexagonal nut: \emph{(1)} the centred input, \emph{(2)} its polar image ($r$ vs.\ $\theta$), \emph{(3)} the angular FFT magnitude $|\mathcal{F}(r,k)|$, and \emph{(4)} the resulting 64-dimensional feature vector that is fed to the MLP (mean-over-$r$ and max-over-$r$ statistics per harmonic $k$). This vector is the rotation-invariant ``symmetry fingerprint'' of the object; the dotted markers highlight the $k{=}0$ (DC), $k{=}4$ and $k{=}6$ harmonics.}
\label{fig:featurevec}
\end{figure*}

\subsection{Stage 4: MLP Classifier}
\label{subsec:stage4}

The spectral feature vector is processed by a three-layer MLP:
\begin{align}
  \mathbf{h}_1 &= \text{ReLU}\!\left(\text{BN}\!\left(\mathbf{W}_1 \mathbf{f} + \mathbf{b}_1\right)\right),\quad \text{Dropout}(0.3),\\
  \mathbf{h}_2 &= \text{ReLU}\!\left(\text{BN}\!\left(\mathbf{W}_2 \mathbf{h}_1 + \mathbf{b}_2\right)\right),\quad \text{Dropout}(0.2),\\
  \hat{\mathbf{y}} &= \mathbf{W}_3 \mathbf{h}_2 + \mathbf{b}_3,
\end{align}
with hidden dimensions 64 and 32. All linear layers use Kaiming normal initialisation \citep{he2015delving}. Batch normalisation \citep{ioffe2015batch} stabilises training on the small dataset.

\textbf{Trainable parameters: 6,564.}

\subsection{Parameter Analysis}
\label{subsec:params}

Table~\ref{tab:params} compares the parameter count of S2P-Net with the CNN baseline (Section~\ref{sec:experiments}). The $323\times$ reduction is not merely a design choice but a direct consequence of the theoretical framework: because rotation invariance is achieved analytically in stages 1--3, the classifier sees only a 64-dimensional, rotation-free feature vector regardless of the input resolution. A conventional CNN must encode rotation invariance implicitly in its weights, requiring proportionally more capacity.

\begin{table}[htbp]
\centering
\caption{Trainable parameter count comparison.}
\label{tab:params}
\begin{tabular}{lrr}
\toprule
Module & S2P-Net & SimpleCNN \\
\midrule
Feature extractor & 0 & 2,113,120 \\
Classifier        & 6,564 & 8,196 \\
\midrule
\textbf{Total}    & \textbf{6,564} & \textbf{2,121,316} \\
Ratio             & $1\times$ & $323\times$ \\
\bottomrule
\end{tabular}
\end{table}

\section{Experiments}
\label{sec:experiments}

\subsection{Dataset}
\label{subsec:dataset}

We collected a dataset of four industrial object categories relevant to robotic pick-and-place tasks:
\begin{itemize}[leftmargin=1.2em]
  \item \textbf{Mutter} (hexagonal nut, 6-fold symmetry): 15 images
  \item \textbf{Stecker} (electrical connector, 1-fold symmetry): 20 images
  \item \textbf{Unterlegscheibe} (circular washer, $\infty$-fold symmetry): 18 images
  \item \textbf{Würfel} (cube face, 4-fold symmetry): 27 images
\end{itemize}
Total: 80 images. All images were captured at $1280 \times 960$ pixels with a fixed overhead camera. Objects were placed on a uniform background. Preprocessing applied Otsu thresholding for background removal, followed by morphological dilation to fill gaps, bounding-box extraction, and centred resizing to $128 \times 128$ pixels with greyscale conversion.

\subsection{Training Setup}
\label{subsec:training}

\paragraph{Data splits.}
Two experimental splits are used. In the \emph{Low-Data} experiment, exactly 3 images per class are used for training (12 images total); the remaining 68 images form the held-out test set (12 Mutter, 17 Stecker, 15 Unterlegscheibe, 24 Würfel). In the \emph{Full-Data} experiment, a 75\%/25\% stratified split yields approximately 62 training and 18 test images.

\paragraph{Augmentation.}
Training images are augmented on-the-fly by a factor of $50\times$ per epoch. For the \emph{Low-Data} experiment, augmentation consists of brightness jitter ($\pm35\%$), Gaussian noise ($\sigma \leq 0.04$), scaling ($[0.88, 1.12]$), translation ($\pm6\%$), and mild contrast shift---no rotations. For the \emph{Full-Data} experiment, uniform random rotation in $[0^\circ, 360^\circ]$ is additionally applied. Test images are \emph{not} augmented; the evaluation protocol is described in Section~\ref{subsec:eval}.

\paragraph{Optimisation.}
Both models are trained with:
\begin{itemize}[leftmargin=1.2em]
  \item Optimiser: AdamW \citep{loshchilov2019decoupled}, $\eta = 10^{-3}$, weight decay $10^{-3}$
  \item Loss: Focal Loss \citep{lin2017focal} ($\gamma = 2$) for S2P-Net;
        cross-entropy for SimpleCNN
  \item Scheduler: Cosine Annealing with Warm Restarts \citep{loshchilov2017sgdr},
        $T_0 = 20$ epochs, $T_\mathrm{mult} = 2$
  \item Batch size: 16; maximum epochs: 200; early stopping patience: 30
  \item Mixed-precision training (FP16) with gradient scaling (GPU only)
\end{itemize}

\subsection{Baseline Architecture}
\label{subsec:baseline}

The CNN baseline (SimpleCNN) consists of three convolutional blocks, each with a $3\times3$ convolution, batch normalisation, ReLU activation, and $2\times2$ max pooling. Channel widths progress $1 \to 16 \to 32 \to 64$. After three pooling steps, the feature map is $16 \times 16$, which is flattened and classified by a two-layer MLP with 128 hidden units and 40\% dropout. This architecture has 2,121,316 parameters and serves as a representative standard CNN for the same task.

\subsection{Evaluation Protocol}
\label{subsec:eval}

After training, both models were evaluated at 12 equally-spaced rotation angles $\{0^\circ, 30^\circ, \ldots, 330^\circ\}$. At each angle $\phi$, every test image was rotated by $\phi$ about its centre using bilinear interpolation before inference. No rotation information was provided to the models; no additional augmentation was applied to the test images. The per-angle accuracy is computed over the complete held-out test set ($N = 68$ images for the Low-Data experiment, $N = 18$ for the Full-Data experiment). This protocol directly measures whether each model has achieved orientation-independent recognition or merely memorised the training orientations.

\section{Results}
\label{sec:results}

\subsection{Full-Data Training with Rotation Augmentation}
\label{subsec:full_data}

When rotation augmentation is included during training, both models converge to 100\% per-angle accuracy across all 12 test angles (Table~\ref{tab:full_data_angles}). This confirms that (a) the task is solvable with sufficient data, (b) both architectures have enough capacity to learn the four classes, and (c) the evaluation protocol is sound. Training curves for this experiment are shown in Figure~\ref{fig:training_full}.

\begin{table}[htbp]
\centering
\caption{Per-angle accuracy (\%) under full-data training with rotation augmentation. Both models achieve perfect accuracy at all orientations ($N=18$ test images).}
\label{tab:full_data_angles}
\begin{tabular}{lcc}
\toprule
Angle & S2P-Net & SimpleCNN \\
\midrule
$0^\circ$   & 100.0 & 100.0 \\
$30^\circ$  & 100.0 & 100.0 \\
$60^\circ$  & 100.0 & 100.0 \\
$90^\circ$  & 100.0 & 100.0 \\
$120^\circ$ & 100.0 & 100.0 \\
$150^\circ$ & 100.0 & 100.0 \\
$180^\circ$ & 100.0 & 100.0 \\
$210^\circ$ & 100.0 & 100.0 \\
$240^\circ$ & 100.0 & 100.0 \\
$270^\circ$ & 100.0 & 100.0 \\
$300^\circ$ & 100.0 & 100.0 \\
$330^\circ$ & 100.0 & 100.0 \\
\midrule
\textbf{Mean} & \textbf{100.0} & \textbf{100.0} \\
\textbf{Std}  & 0.0 & 0.0 \\
\bottomrule
\end{tabular}
\end{table}

\subsection{Low-Data Setting without Rotation Augmentation}
\label{subsec:few_shot}

The critical experiment measures what happens when rotation augmentation is withheld---the exact scenario one encounters when a new part type is introduced to a system with limited labeled data. Only 3 images per class (12 total) are used for training; the remaining 68 images serve as the test set. Table~\ref{tab:few_shot_angles} and Figure~\ref{fig:rotation_few_shot} present the per-angle accuracy for both models.

\begin{table}[htbp]
\centering
\caption{Per-angle accuracy (\%) in the low-data setting without rotation augmentation
($N=68$ held-out test images: 12 Mutter, 17 Stecker, 15 Unterlegscheibe, 24 Würfel).
S2P-Net maintains stable accuracy; the CNN degrades catastrophically at
$120^\circ$--$210^\circ$.}
\label{tab:few_shot_angles}
\begin{tabular}{lccc}
\toprule
Angle & S2P-Net & SimpleCNN & $\Delta$ \\
\midrule
$0^\circ$   & 72.1 & 89.7 & $-17.6$ \\
$30^\circ$  & 73.5 & 89.7 & $-16.2$ \\
$60^\circ$  & 72.1 & 76.5 & $-4.4$  \\
$90^\circ$  & 70.6 & 64.7 & $+5.9$  \\
$120^\circ$ & 69.1 & 50.0 & $+19.1$ \\
$150^\circ$ & 70.6 & 45.6 & $+25.0$ \\
$180^\circ$ & \textbf{75.0} & \textbf{19.1} & $\mathbf{+55.9}$ \\
$210^\circ$ & 69.1 & 27.9 & $+41.2$ \\
$240^\circ$ & 70.6 & 36.8 & $+33.8$ \\
$270^\circ$ & 70.6 & 70.6 & $0.0$   \\
$300^\circ$ & 70.6 & 73.5 & $-2.9$  \\
$330^\circ$ & 70.6 & 76.5 & $-5.9$  \\
\midrule
\textbf{Mean} & \textbf{71.2} & \textbf{60.0} & $\mathbf{+11.2}$ \\
\textbf{Std}  & \textbf{1.6}  & \textbf{22.9} & --- \\
\bottomrule
\end{tabular}
\end{table}

Several patterns are immediately apparent:

\textbf{S2P-Net maintains a flat accuracy profile.} Across all 12 angles, S2P-Net varies between $69.1\%$ (at $120^\circ$ and $210^\circ$) and $75.0\%$ (at $180^\circ$), a range of only $5.9$ pp with standard deviation $1.6\%$. This near-constant profile is the empirical signature of true rotation invariance: performance does not depend on which angle was seen at training time.

\textbf{The CNN degrades catastrophically near $180^\circ$.} The CNN performs well at $0^\circ$ and $30^\circ$ (89.7\%), because those orientations closely resemble the (non-rotated) training distribution. As the test angle diverges from the training distribution, performance falls sharply, reaching a minimum of $19.1\%$ at $180^\circ$---barely above the $25\%$ random-chance baseline for four classes.\footnote{A network that scores $89.7\%$ upright and $19.1\%$ upside-down has not really learned to recognise the parts; it has learned which way is up.} The CNN partially recovers near $270^\circ$ (70.6\%), consistent with the fact that a $270^\circ$ rotation is equivalent to a $90^\circ$ rotation, which preserves many visual features.

\textbf{Summary statistics.} S2P-Net achieves a mean accuracy of $71.2\%$ with $\sigma = 1.6\%$. The CNN achieves $60.0\%$ with $\sigma = 22.9\%$. The $14\times$ higher standard deviation of the CNN directly quantifies its orientation sensitivity.

\begin{figure}[t]
\centering
\includegraphics[width=\linewidth]{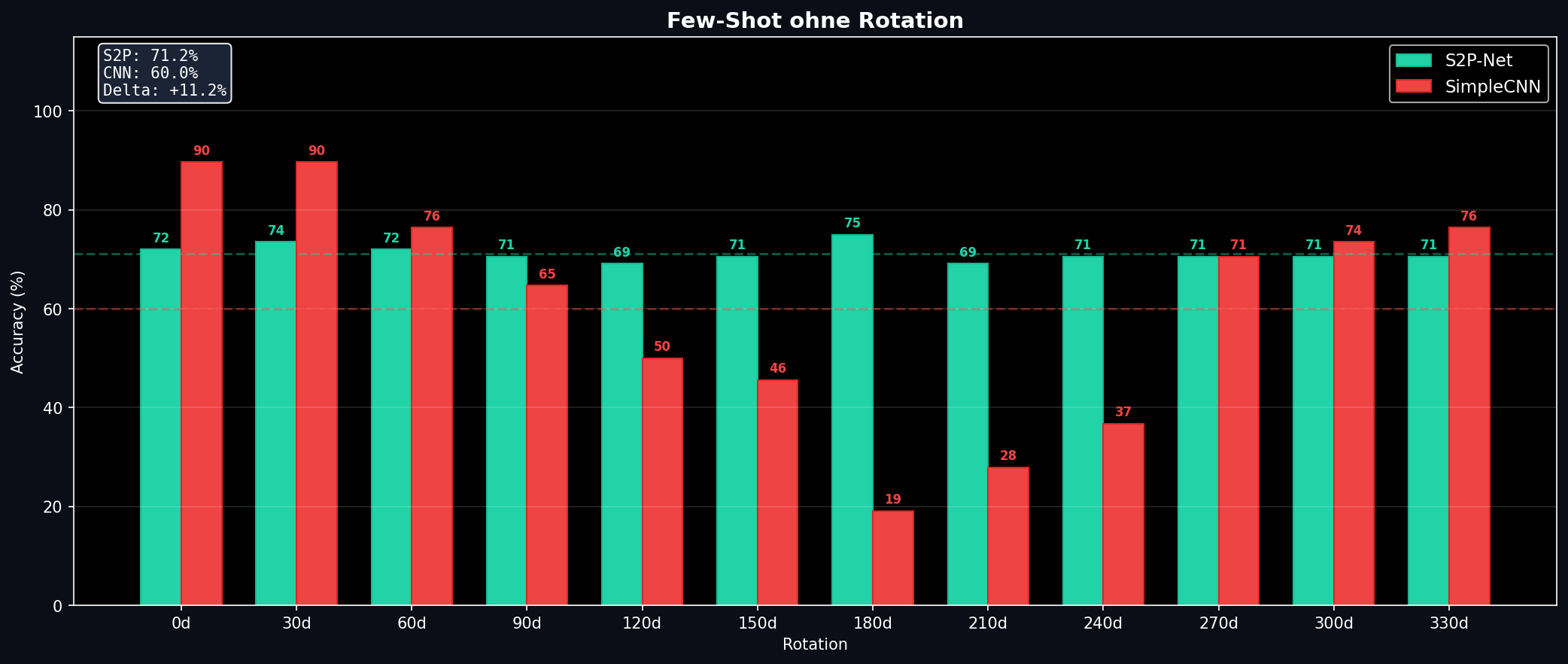}
\caption{Per-angle classification accuracy in the low-data setting without rotation augmentation. S2P-Net (teal) maintains a stable profile across all orientations. The CNN (red) performs well near the training distribution but collapses to near-chance at $180^\circ$.}
\label{fig:rotation_few_shot}
\end{figure}

\begin{figure}[t]
\centering
\includegraphics[width=\linewidth]{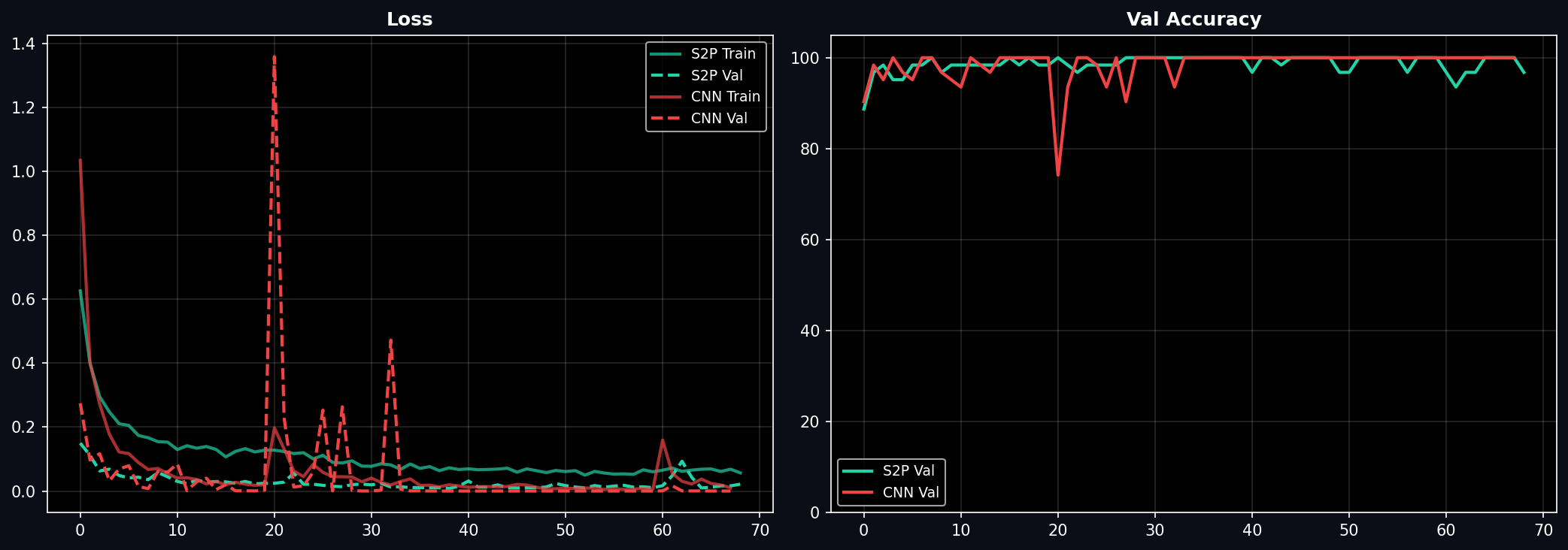}
\caption{Training and validation curves for the full-data experiment with rotation augmentation. Both models converge to perfect validation accuracy.}
\label{fig:training_full}
\end{figure}

\begin{figure}[t]
\centering
\includegraphics[width=\linewidth]{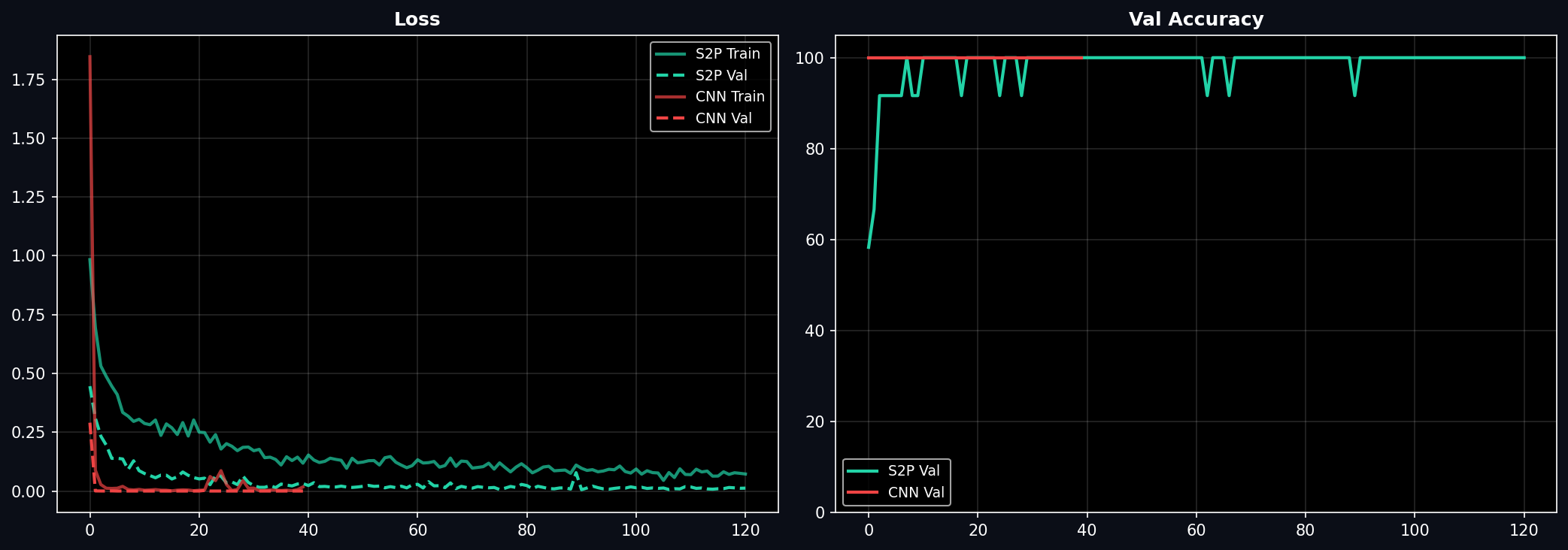}
\caption{Training curves for the low-data experiment without rotation augmentation. S2P-Net converges stably; the CNN shows higher variance due to the extremely limited training set (3 images per class).}
\label{fig:training_few_shot}
\end{figure}

\begin{figure}[t]
\centering
\includegraphics[width=\linewidth]{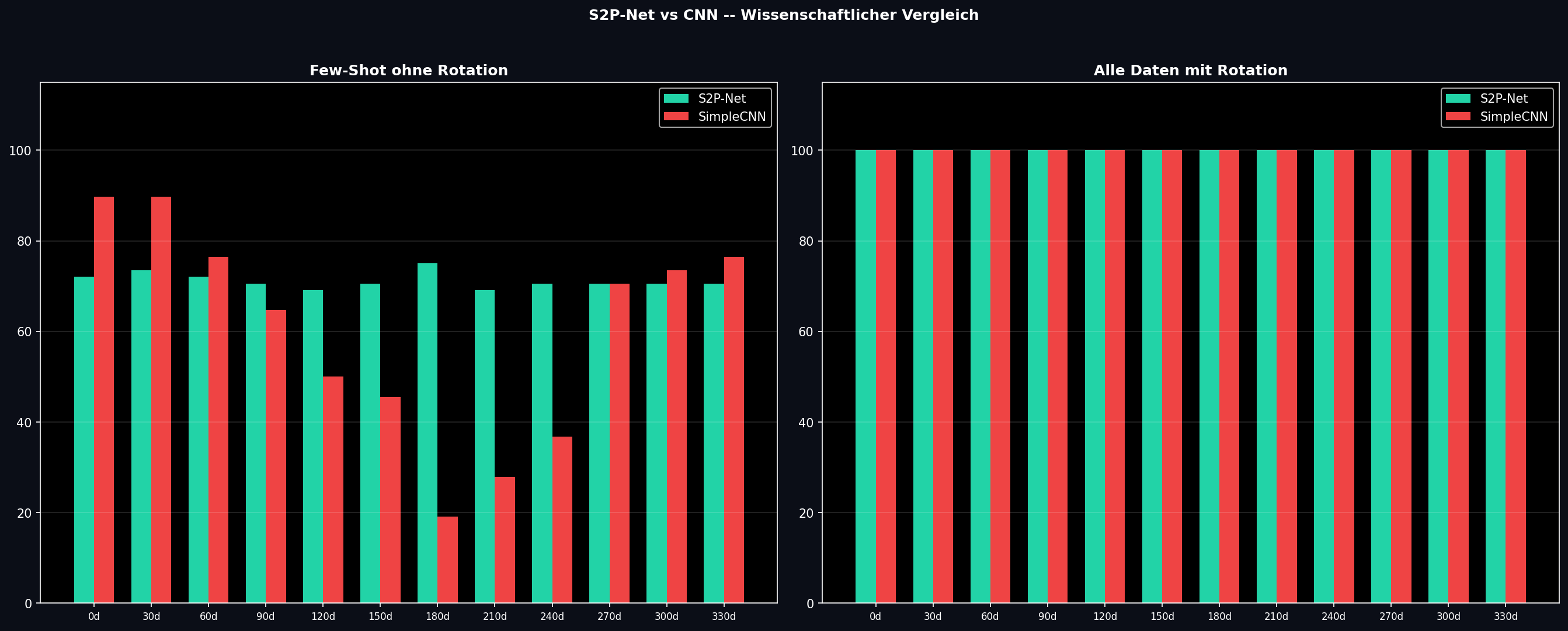}
\caption{Side-by-side comparison of per-angle accuracy for both experiments. Left: low-data setting without rotation augmentation (S2P-Net dominates at mid-angles). Right: full-data with rotation augmentation (both models perfect).}
\label{fig:full_comparison}
\end{figure}

\subsection{Scale invariance via log-polar sampling}
\label{subsec:scale_results}

To test Proposition~\ref{prop:scale} we evaluate three models under a unified protocol: the standard linear-polar S2P-Net, the log-polar S2P-Net, and the SimpleCNN baseline. All three are trained on a $75/25$ split of the four-class dataset with rotation ($0^\circ$--$360^\circ$) and photometric augmentation, but \emph{without} any scale or translation augmentation, so that scale robustness must arise by construction rather than from training exposure. To control for the stochasticity criticised in Section~\ref{sec:discussion}, every result in this section and the next is averaged over \textbf{five independent seeds} (with re-randomised train/test splits); we report the mean and, as $\sigma$, the spread across the sweep.

At test time each held-out image is rescaled about its centre by a factor $s\in\{0.6,\dots,1.4\}$ and, at each scale, evaluated at all twelve rotation angles; the reported per-scale accuracy is the mean over these rotations and over the test set. Table~\ref{tab:scale} and Figure~\ref{fig:scale} summarise the outcome.

\begin{table}[htbp]
\centering
\caption{Per-scale accuracy (\%, mean$\pm$std over 5 seeds, each averaged over 12 rotations) under rotation-only training. The log-polar variant is the most scale-stable; the CNN collapses when the object shrinks below its training scale.}
\label{tab:scale}
\small
\begin{tabular}{lccc}
\toprule
Scale & S2P-Log & S2P-Linear & CNN \\
\midrule
$0.6\times$ & \textbf{72.6}\,$\pm$\,10.8 & 52.2\,$\pm$\,2.7 & 20.8\,$\pm$\,2.4 \\
$0.7\times$ & \textbf{72.2}\,$\pm$\,11.7 & 61.6\,$\pm$\,7.3 & 37.4\,$\pm$\,6.6 \\
$0.8\times$ & 71.8\,$\pm$\,11.1 & 60.7\,$\pm$\,12.9 & \textbf{87.3}\,$\pm$\,4.0 \\
$0.9\times$ & 79.1\,$\pm$\,9.5 & 74.6\,$\pm$\,15.2 & \textbf{97.6}\,$\pm$\,2.8 \\
$1.0\times$ & 93.5\,$\pm$\,5.4 & 95.6\,$\pm$\,5.4 & \textbf{99.8}\,$\pm$\,0.2 \\
$1.1\times$ & 91.8\,$\pm$\,6.7 & 94.1\,$\pm$\,7.3 & \textbf{98.4}\,$\pm$\,1.5 \\
$1.2\times$ & 94.0\,$\pm$\,5.1 & \textbf{95.9}\,$\pm$\,5.0 & 97.6\,$\pm$\,3.7 \\
$1.3\times$ & 92.1\,$\pm$\,5.0 & \textbf{93.9}\,$\pm$\,4.9 & 89.9\,$\pm$\,3.8 \\
$1.4\times$ & \textbf{89.9}\,$\pm$\,3.3 & 89.8\,$\pm$\,5.9 & 81.0\,$\pm$\,3.1 \\
\midrule
\textbf{Mean} & \textbf{84.1} & 79.8 & 78.9 \\
\textbf{Std}  & \textbf{12.5} & 18.6 & 27.7 \\
\bottomrule
\end{tabular}
\end{table}

Three observations stand out. First, the log-polar variant has both the \emph{highest} mean accuracy ($84.1\%$) and the \emph{lowest} spread ($\sigma=12.5\%$) of the three models---the empirical signature of scale invariance, exactly mirroring the flat-profile result for rotation. Second, the CNN reproduces its rotation pathology in the scale domain: it is near-perfect at and above its training scale ($99.8\%$ at $1.0\times$) but collapses to $20.8\%$ at $0.6\times$, because a strongly shrunken object lies far outside the training distribution. Third, the advantage of log-polar over linear sampling is concentrated in the down-scaling regime ($s\le 0.8$: $+19.6$ to $+20.4$\,pp), where the shrunken object's energy migrates toward the centre and stays within the radial window---precisely the band predicted by Proposition~\ref{prop:scale}. For up-scaling, all models eventually decline as the object is clipped by the image border, a sensor-field-of-view limit rather than a model limit.

\begin{figure}[t]
\centering
\includegraphics[width=\linewidth]{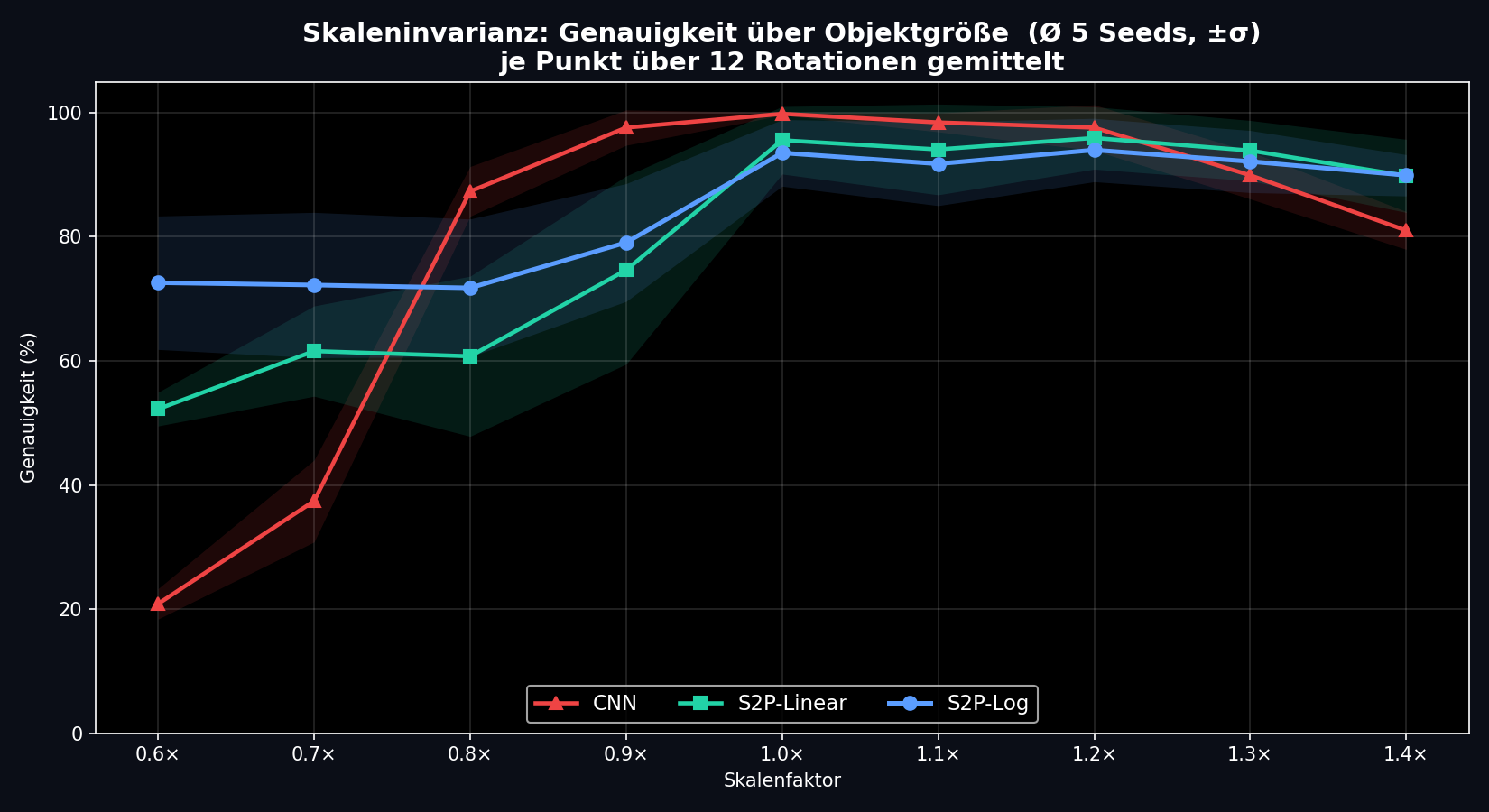}
\caption{Accuracy versus object scale (5 seeds, shaded $\pm\sigma$; each point averaged over 12 rotations). The log-polar S2P-Net (blue) maintains a flat profile, while the CNN (red) collapses for shrunken objects despite being strongest at its training scale.}
\label{fig:scale}
\end{figure}

\subsection{Centering robustness}
\label{subsec:center_results}

S2P-Net's only structural assumption is that the object is centred, since the polar origin is fixed to the image centre. We quantify the cost of violating this assumption by translating each held-out image by an offset of $0$--$20\%$ of the image width (averaged over four cardinal directions), again over five seeds (Table~\ref{tab:center}, Figure~\ref{fig:center}).

\begin{table}[htbp]
\centering
\caption{Per-offset accuracy (\%, mean$\pm$std over 5 seeds, 4 directions) under de-centering. All models degrade; the CNN is the most translation-tolerant, confirming centering as S2P-Net's principal limitation.}
\label{tab:center}
\small
\begin{tabular}{lccc}
\toprule
Offset & S2P-Log & S2P-Linear & CNN \\
\midrule
$0\%$  & 93.3\,$\pm$\,5.4 & 95.6\,$\pm$\,5.4 & \textbf{100.0}\,$\pm$\,0.0 \\
$5\%$  & 86.1\,$\pm$\,4.2 & 84.2\,$\pm$\,6.5 & \textbf{95.8}\,$\pm$\,1.8 \\
$10\%$ & 70.8\,$\pm$\,2.8 & 65.0\,$\pm$\,9.9 & \textbf{83.3}\,$\pm$\,3.0 \\
$15\%$ & 55.3\,$\pm$\,5.5 & 41.4\,$\pm$\,8.9 & \textbf{62.5}\,$\pm$\,2.9 \\
$20\%$ & 40.8\,$\pm$\,7.5 & 38.6\,$\pm$\,9.0 & \textbf{51.1}\,$\pm$\,4.1 \\
\midrule
\textbf{Mean} & 69.3 & 64.9 & \textbf{78.6} \\
\bottomrule
\end{tabular}
\end{table}

Unlike the rotation and scale results, here the CNN is the most robust at every offset, because its convolutional stack is translation-equivariant and only its final fully-connected layer is position-sensitive. S2P-Net, by contrast, recomputes an incorrect polar profile once the object leaves the centre, and its accuracy falls below the CNN's beyond a $5\%$ offset. The effect is real but graceful for small offsets ($\le 5\%$), and the log-polar variant degrades more gently than the linear one (e.g.\ $+13.9$\,pp at $15\%$), because its radial compression weights the well-sampled central region most heavily. This experiment turns the qualitative ``centering requirement'' of prior discussion into a concrete operating bound; Section~\ref{subsec:frontend} shows that a simple centroid front-end removes the dependence entirely.

\begin{figure}[t]
\centering
\includegraphics[width=\linewidth]{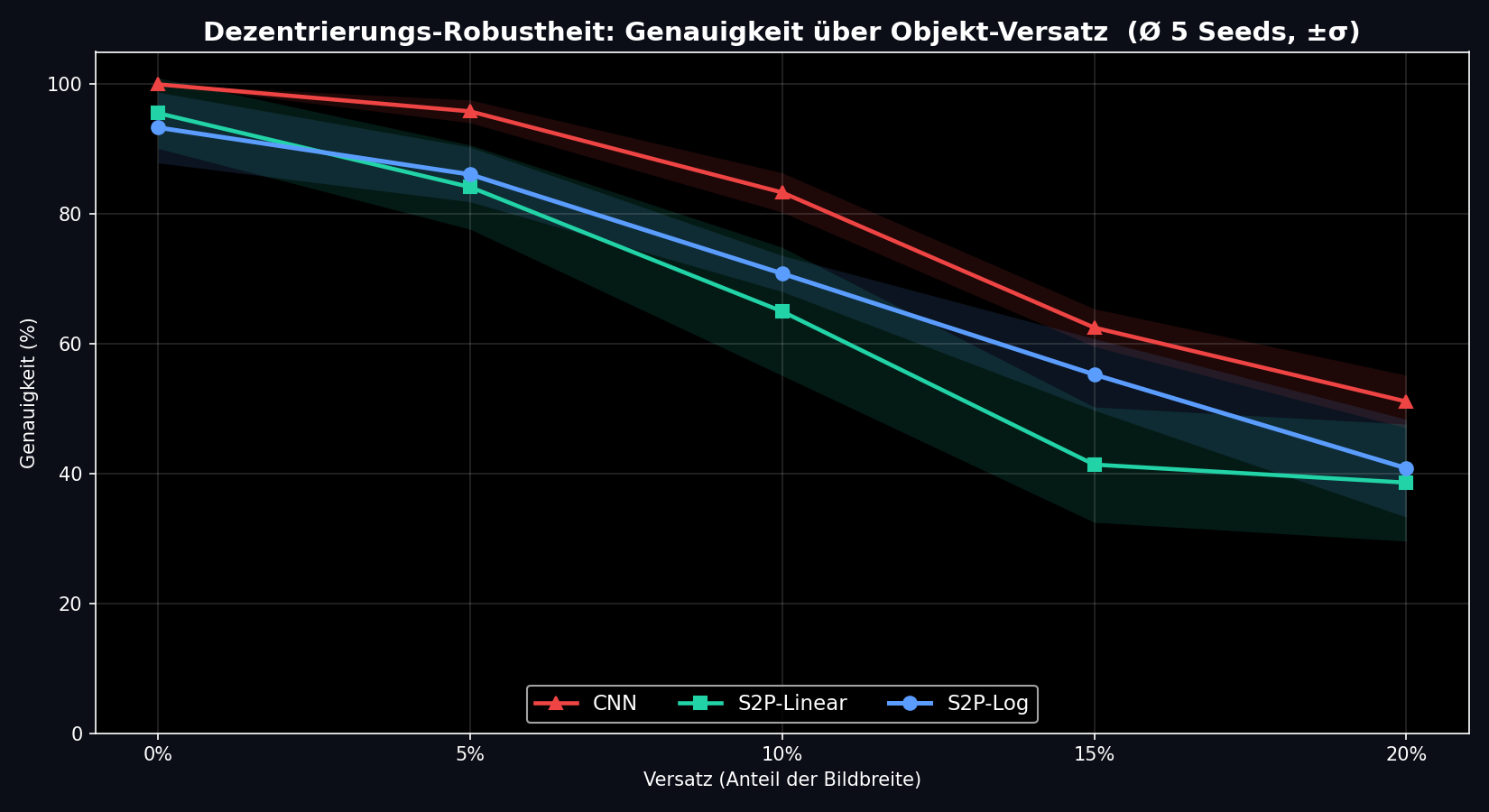}
\caption{Accuracy versus object de-centering (5 seeds, shaded $\pm\sigma$). All models degrade; the translation-tolerant CNN leads, quantifying S2P-Net's centering assumption as its principal limitation.}
\label{fig:center}
\end{figure}

\subsection{A centroid front-end removes the centering dependence}
\label{subsec:frontend}

The centering sensitivity of Section~\ref{subsec:center_results} is not intrinsic to the recognition task---it is an artefact of feeding the network a mis-aligned crop. Since the polar origin only has to coincide with the object centroid, a lightweight detector can restore the assumption before inference. We implement such a front-end with the same Otsu-threshold-plus-image-moments computation already used in preprocessing: it locates the object in the raw frame, crops a square window about its centroid, and resizes to $128\times128$. We then repeat the de-centering protocol on the raw images, where ``$0$--$20\%$ offset'' now denotes a mis-aligned crop window (the failure mode of a naive fixed-window grab), and compare the raw input against the front-end-corrected input.

\begin{table}[htbp]
\centering
\caption{Accuracy (\%, mean$\pm$std over 5 seeds) under a mis-aligned crop, with and without the centroid front-end. The front-end restores the centred-baseline accuracy at every offset.}
\label{tab:frontend}
\small
\begin{tabular}{lccc}
\toprule
Offset & S2P-Log + FE & S2P-Log (raw) & CNN (raw) \\
\midrule
$0\%$  & 93.3\,$\pm$\,5.4 & 93.3\,$\pm$\,5.4 & \textbf{100.0}\,$\pm$\,0.0 \\
$5\%$  & \textbf{93.3}\,$\pm$\,5.4 & 91.1\,$\pm$\,3.0 & 92.5\,$\pm$\,2.6 \\
$10\%$ & \textbf{93.3}\,$\pm$\,5.4 & 72.5\,$\pm$\,2.7 & 84.4\,$\pm$\,1.0 \\
$15\%$ & \textbf{93.3}\,$\pm$\,5.4 & 56.9\,$\pm$\,4.6 & 68.6\,$\pm$\,4.5 \\
$20\%$ & \textbf{93.3}\,$\pm$\,5.4 & 42.8\,$\pm$\,8.0 & 56.7\,$\pm$\,5.2 \\
\midrule
\textbf{Mean} & \textbf{93.3} & 71.3 & 80.4 \\
\bottomrule
\end{tabular}
\end{table}

Table~\ref{tab:frontend} and Figure~\ref{fig:frontend} show that the front-end makes S2P-Net's accuracy \emph{independent of the crop offset}: it holds a flat $93.3\%$---its centred baseline---across the entire sweep, whereas the raw input collapses to $42.8\%$ at a $20\%$ offset. Beyond a $5\%$ offset the corrected S2P-Net also overtakes the CNN ($93.3\%$ vs.\ $84.4\%$ at $10\%$, $93.3\%$ vs.\ $56.7\%$ at $20\%$), because the CNN receives the still-mis-aligned crop. The de-centering limitation is therefore reduced to the reliability of object detection, which on single-object, uniform-background industrial scenes is essentially solved. This result removes the principal practical objection to S2P-Net at negligible cost.

\begin{figure}[t]
\centering
\includegraphics[width=\linewidth]{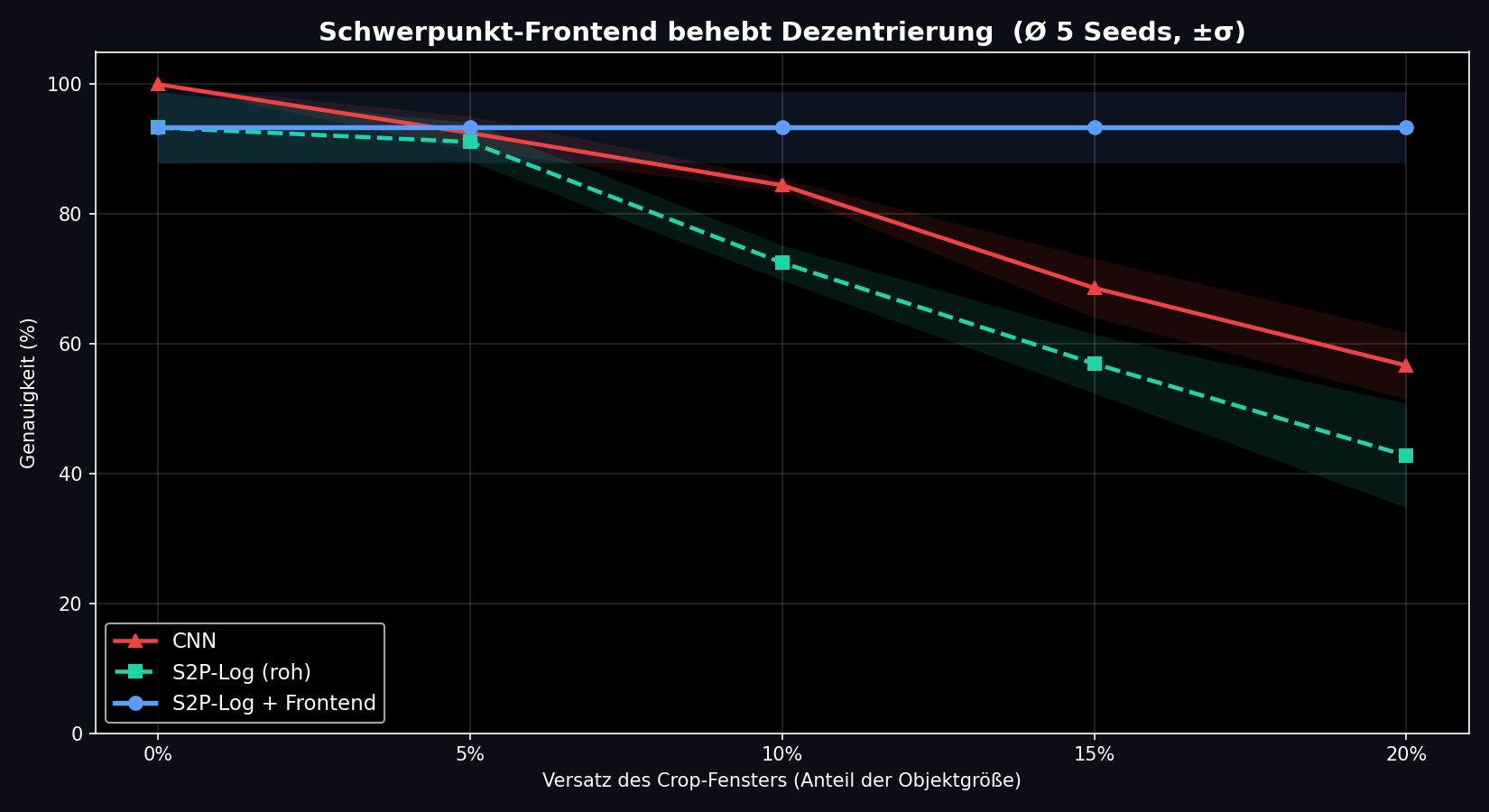}
\caption{Centroid front-end (blue) versus raw mis-aligned input (5 seeds, shaded $\pm\sigma$). With detection, S2P-Net's accuracy is flat across all offsets and surpasses the CNN once the object is meaningfully off-centre.}
\label{fig:frontend}
\end{figure}

\subsection{Robustness to image degradations}
\label{subsec:robustness}

The invariances established so far are \emph{geometric}. A natural question is whether the spectral representation also confers robustness to \emph{photometric and structural} corruptions encountered in practice. We test three, each as a severity sweep applied at test time only (training saw rotation and mild photometric augmentation with noise $\sigma\le0.04$, but none of these stronger corruptions), again over five seeds and averaged over twelve rotations: additive Gaussian noise, square occlusion of a given area fraction, and linear motion blur. Table~\ref{tab:robust} reports the mean accuracy over each severity sweep; Figure~\ref{fig:robust} shows the full curves.

\begin{table}[htbp]
\centering
\caption{Robustness: mean accuracy (\%) over each degradation sweep (5 seeds, averaged over 12 rotations). Best per row in bold. The spectral models are competitive on noise but markedly weaker under occlusion.}
\label{tab:robust}
\small
\begin{tabular}{lccc}
\toprule
Degradation & S2P-Log & S2P-Linear & CNN \\
\midrule
Gaussian noise   & 45.5 & \textbf{51.5} & 47.3 \\
Occlusion        & 54.5 & 63.3 & \textbf{76.7} \\
Motion blur      & 78.5 & 73.2 & \textbf{87.1} \\
\bottomrule
\end{tabular}
\end{table}

The result is informative precisely because it is \emph{not} a clean win for S2P-Net. Under additive noise the three models are comparable, with the linear-polar variant marginally ahead ($51.5\%$ vs.\ $47.3\%$ for the CNN), as the averaging in the spectral pooling partially suppresses zero-mean noise; beyond $\sigma=0.1$ all three collapse toward chance. Under \emph{occlusion} the CNN is clearly the most robust at every severity ($76.7\%$ mean vs.\ $54.5\%$ for S2P-Log), and under motion blur it again leads ($87.1\%$), although the log-polar variant degrades most gracefully among the spectral models. The reason is structural: a \emph{localised} corruption such as occlusion alters the angular profile only over a limited arc, but the 1-D FFT spreads that local change across \emph{all} frequency bins, perturbing the entire magnitude spectrum; a convolutional network, whose receptive fields remain local, loses only the features overlapping the occluded region. S2P-Net thus trades robustness to local corruption for analytic geometric invariance---a profile complementary to that of CNNs.

\begin{figure*}[t]
\centering
\includegraphics[width=0.32\textwidth]{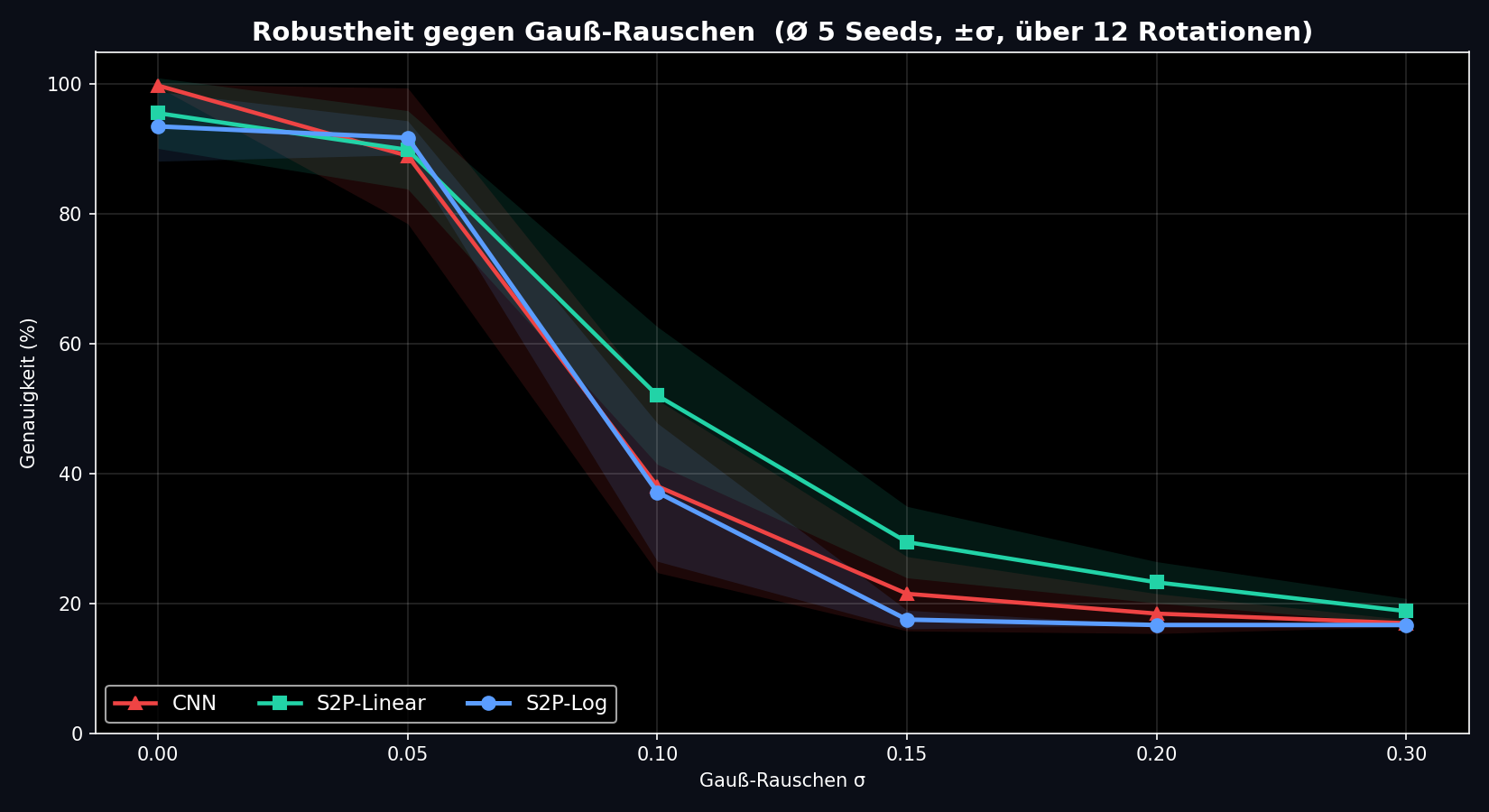}\hfill
\includegraphics[width=0.32\textwidth]{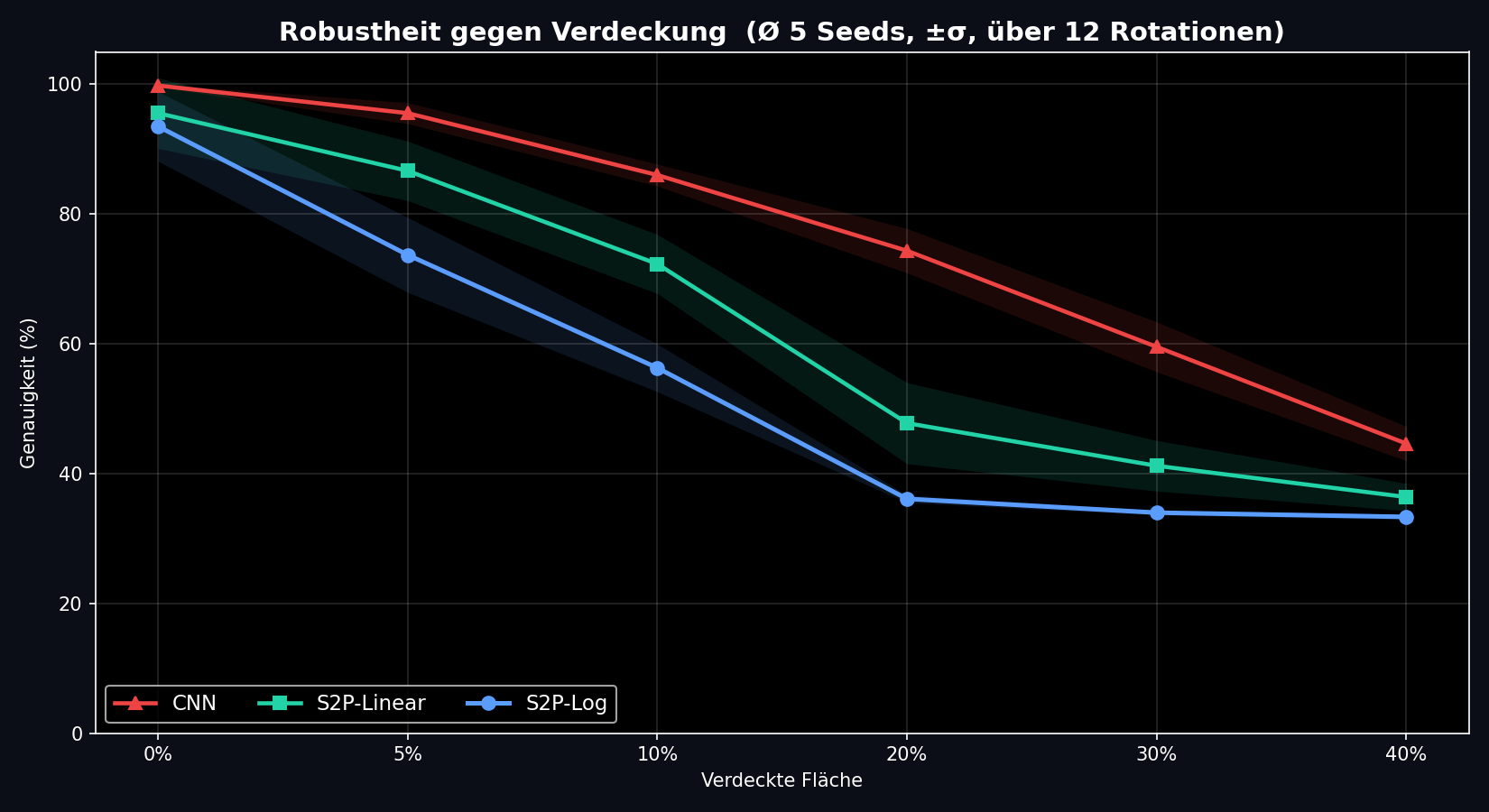}\hfill
\includegraphics[width=0.32\textwidth]{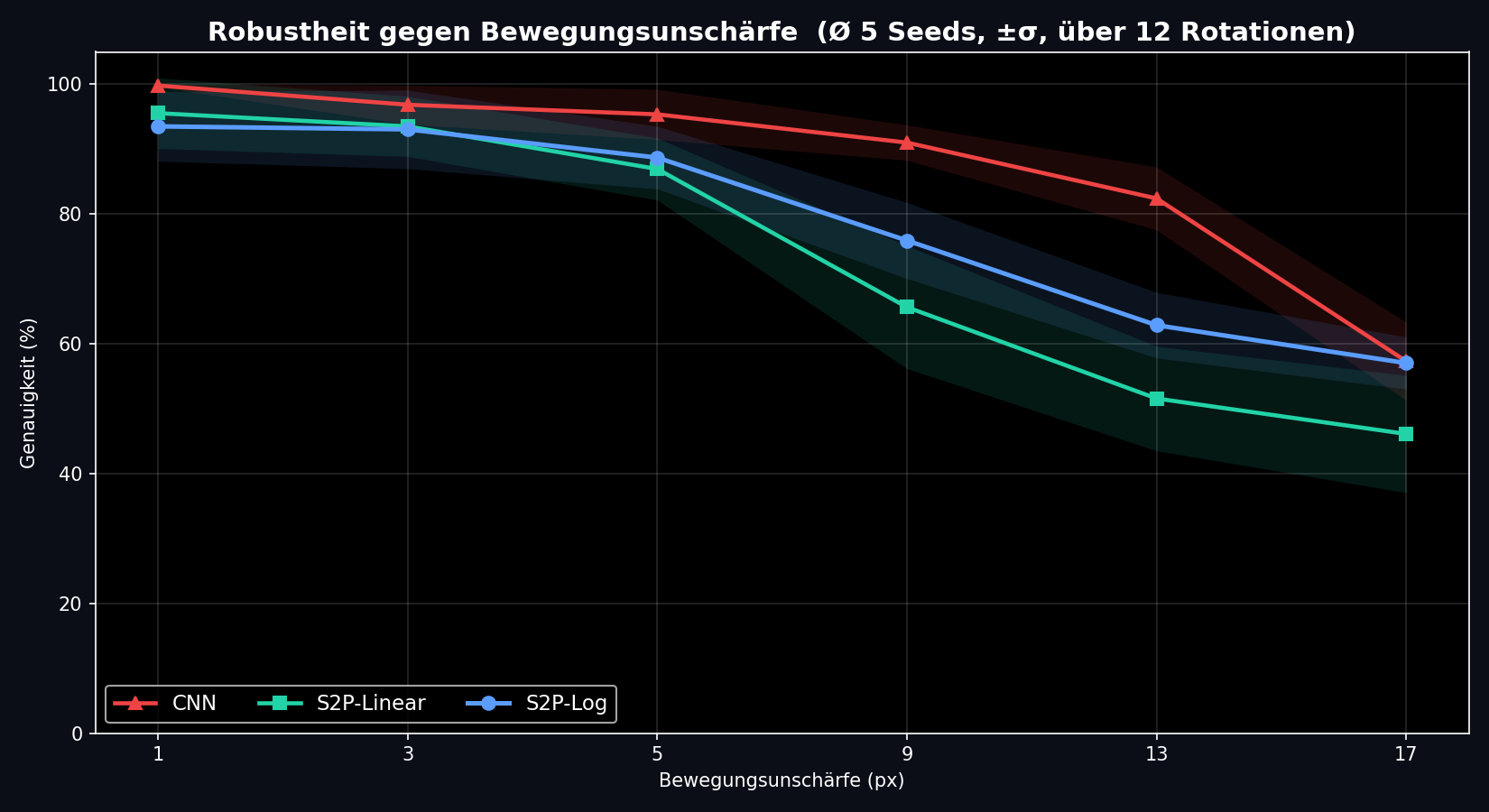}
\caption{Robustness to image degradations (5 seeds, shaded $\pm\sigma$, averaged over 12 rotations). Left: additive Gaussian noise (comparable, S2P-Linear marginally ahead). Centre: occlusion (CNN clearly most robust). Right: motion blur (CNN leads; log-polar most graceful among spectral models).}
\label{fig:robust}
\end{figure*}

\section{Discussion}
\label{sec:discussion}

\subsection{Why S2P-Net does not reach 100\% in the low-data setting}
Theorem~\ref{thm:invariance} guarantees rotation invariance of the feature representation, but not perfect classification. The $\approx71\%$ accuracy ceiling reflects factors unrelated to rotation: inter-class similarity at the spectral level (a circular washer and a cube face can share low-frequency harmonics), within-class variation in scale and exact centring, and the limited capacity of the 6,564-parameter classifier trained on only 3 images per class. Importantly, this ceiling is \emph{consistent across angles}, confirming that the residual errors are classification errors, not orientation errors.

\subsection{CNN performance at \texorpdfstring{$0^\circ$}{0 degrees} vs.\ S2P-Net}
At $0^\circ$ and $30^\circ$, the CNN outperforms S2P-Net (89.7\% vs.\ 72.1\%). This is expected: with only 3 training images per class and no rotation augmentation, the CNN fits the specific near-upright orientations seen during training and becomes a highly specialised detector for those poses. S2P-Net sacrifices this specialisation in exchange for uniform performance across all orientations. In practice, if the operating orientation is fully controlled and fixed, the CNN would be preferable. For unconstrained industrial settings, S2P-Net's flat profile is strongly advantageous.

\subsection{Symmetry classes and spectral discriminability}
The harmonic decomposition is physically meaningful. Hexagonal nuts ($k=6$), cubes ($k=4$), and washers ($k=0, 2, 4, \ldots$) have distinct symmetry profiles that appear directly in the magnitude spectrum. Electrical connectors with asymmetric pin arrangements have non-zero energy at $k=1$. This interpretability stands in contrast to the learned but opaque features of conventional CNNs.

\subsection{Limitations}
\textbf{Centring requirement.} Polar coordinates are computed relative to the image centre, so S2P-Net assumes the object is centred. Section~\ref{subsec:center_results} quantifies the cost of violating this assumption: accuracy is preserved for small offsets ($\le5\%$ of the image width) but falls below the translation-tolerant CNN beyond that, reaching near-chance at a $20\%$ offset. In our application centring is enforced by the preprocessing pipeline (Otsu thresholding + bounding-box crop); Section~\ref{subsec:frontend} shows that wiring this same centroid computation in as an explicit front-end restores the centred-baseline accuracy at every offset, reducing the limitation to the reliability of object detection.

\textbf{Scale sensitivity.} The original linear-polar formulation has no built-in scale invariance. Section~\ref{subsec:scale_results} shows that a log-polar transform \citep{sosnovik2020scale}---a one-line change to the radial sampling, with no added parameters---makes the model the most scale-stable of the three tested (mean $84.1\%$, $\sigma=12.5\%$), at the cost of a small accuracy reduction at the nominal scale. The guarantee is window-limited (Proposition~\ref{prop:scale}): it holds while the object remains within the sampled radial band and degrades once the object is clipped by the image border.

\textbf{Dataset size.} Our evaluation used 80 images across 4 classes. While sufficient to demonstrate the rotation-invariance advantage, a larger dataset would be needed to assess performance on more fine-grained distinctions and higher intra-class variability.

\textbf{In-plane rotation only.} The Fourier rotation theorem applies to 2-D in-plane rotations. S2P-Net does not handle 3-D pose variation (e.g., objects tilting out of plane).

\textbf{Sensitivity to local corruption.} Section~\ref{subsec:robustness} shows that S2P-Net is less robust to occlusion (and, to a lesser degree, blur) than the CNN, because the angular FFT converts any localised disturbance into a global perturbation of the magnitude spectrum. This is the same root cause as the centering sensitivity: the method assumes a complete, centred angular profile. It is the natural price of a global, analytic representation and suggests that S2P-Net is best deployed where objects are fully visible---or in a hybrid with a local-feature branch for partial-occlusion settings.

\textbf{Experimental repetition.} The original rotation experiments (Sections~\ref{subsec:full_data}--\ref{subsec:few_shot}) report a single training run per model; due to the stochasticity of initialisation and augmentation, repeated runs may shift the absolute values. The scale and centering experiments (Sections~\ref{subsec:scale_results}--\ref{subsec:center_results}) address this by averaging over five seeds with re-randomised splits, and the qualitative conclusions there are stable across seeds. Extending the five-seed protocol to the rotation experiments on a larger dataset remains future work.

\section{Conclusion}
\label{sec:conclusion}

We have presented S2P-Net, a rotation-invariant image classifier that achieves its invariance through a three-stage, parameter-free feature extraction pipeline grounded in the Fourier shift theorem. By transforming images to polar coordinates and analysing the angular Fourier magnitude spectrum, S2P-Net produces a 64-dimensional feature vector that is provably unchanged by any in-plane rotation. A 6,564-parameter MLP then performs classification on this invariant representation.

In a low-data industrial recognition scenario with only 12 training images (3 per class) and no rotation augmentation, S2P-Net maintains $71.2\%$ accuracy with a standard deviation of $1.6\%$ across 12 test orientations, while a standard CNN baseline averages $60.0\%$ with a standard deviation of $22.9\%$ and collapses to $19.1\%$ at $180^\circ$. When sufficient augmented data is provided, both models achieve perfect accuracy, confirming that the theoretical framework is sound and that S2P-Net's advantage is specifically in the data-limited regime.

The key take-away is practical: \emph{mathematical inductive bias can substitute for data}. In applications where collecting rotation-augmented training data is expensive or impractical---new product lines in manufacturing, surgical instrument recognition, aerial target classification---architectures that encode known invariances analytically will consistently outperform those that must learn them empirically.

Building on the two extensions introduced here---log-polar sampling, which adds scale invariance at zero parameter cost, and a centroid front-end, which removes the centering dependence---future work will replace the moment-based detector with a learned one for cluttered multi-object scenes, couple S2P-Net features with metric-learning objectives for few-shot generalisation to unseen classes, and deploy the system in a full robotic pick-and-place loop with an ESP32-controlled servo arm.

\section*{Acknowledgements}
This project was carried out independently as part of the youth-science competition, on a single consumer GPU. I am grateful to my family and teachers for their patience and encouragement through the many evenings it required, and to the open-source community behind PyTorch, OpenCV and NumPy, whose tools let one student test ideas that not long ago would have needed a laboratory. I also thank the nuts, washers, cubes and connectors that posed, without complaint, at every conceivable angle. The implementation and dataset are available from the author on request.

\bibliographystyle{abbrvnat}
\bibliography{references}

\end{document}